\documentclass[letterpaper]{article} % DO NOT CHANGE THIS
\usepackage{aaai24}  % DO NOT CHANGE THIS
\usepackage{times}  % DO NOT CHANGE THIS
\usepackage{helvet}  % DO NOT CHANGE THIS
\usepackage{courier}  % DO NOT CHANGE THIS
\usepackage[hyphens]{url}  % DO NOT CHANGE THIS
\usepackage{graphicx} % DO NOT CHANGE THIS
\usepackage{natbib}  % DO NOT CHANGE THIS AND DO NOT ADD ANY OPTIONS TO IT
\usepackage{caption} % DO NOT CHANGE THIS AND DO NOT ADD ANY OPTIONS TO IT
\usepackage{algorithm}
\usepackage{color}
\usepackage{amsmath}
\usepackage{amssymb}

\usepackage{multirow}
\usepackage{anyfontsize}
\usepackage{sidecap}
\usepackage{graphicx}
\usepackage{amsmath}
\usepackage{amssymb}
\usepackage{booktabs}
\usepackage{xcolor}
\usepackage{soul}
\usepackage{colortbl}
\usepackage{algpseudocode,algorithm,algorithmicx}

\usepackage{amsthm}
\usepackage{stmaryrd}
\usepackage{multirow}
\usepackage{array}
\usepackage{MnSymbol}

\newtheorem{rem}{Remark}

\newtheorem{ass}{Assumption}[]

\newtheorem{lem}{Lemma}[]
\newtheorem{prop}{Proposition}[]

\usepackage{newfloat}
\usepackage{listings}
\DeclareCaptionStyle{ruled}{labelfont=normalfont,labelsep=colon,strut=off} % DO NOT CHANGE THIS
\floatstyle{ruled}
\newfloat{listing}{tb}{lst}{}
\floatname{listing}{Listing}
\title{
	Diagnosing and Rectifying Fake OOD Invariance: A Restructured \\Causal Approach
}
\author{
    Ziliang Chen\textsuperscript{\rm 1,}\textsuperscript{\rm 3},\ Yongsen Zheng\textsuperscript{\rm 2},\ Zhao-Rong Lai\textsuperscript{\rm 1},\ Quanlong Guan\textsuperscript{\rm 1}\thanks{indicate corresponding author},\ Liang Lin\textsuperscript{\rm 2}
}
\affiliations{
    \textsuperscript{\rm 1}Jinan University, \ 
     \textsuperscript{\rm 2}Sun Yat-sen University, \
      \textsuperscript{\rm 3}Pazhou Lab\\
      \tt\scriptsize c.ziliang@yahoo.com, \tt\scriptsize z.yongsensmile@gmail.com, \tt\scriptsize \{laizhr,Gql\}@jnu.edu.cn, \tt\scriptsize linliang@ieee.org
}

\usepackage{bibentry}
\begin{document}

\maketitle

\begin{abstract}
%	Invariant Causal Prediction (Peters et al., 2016) is a technique for out-of-distribution	generalization which assumes that some aspects of the data distribution vary across	the training set but that the underlying causal mechanisms remain constant.
	
Invariant representation learning (IRL) encourages the prediction from invariant causal features to labels deconfounded \hspace{-0.1em}from the environments, advancing the technical roadmap of out-of-distribution (OOD) generalization. Despite spotlights around, recent theoretical result verified that some causal features recovered by IRLs merely pretend domain-invariantly in the training environments but fail in unseen domains. The \emph{fake invariance} severely endangers OOD generalization since the trustful objective can not be diagnosed and existing causal remedies are invalid to rectify. In this paper, we review a IRL family (InvRat) under the Partially and Fully Informative Invariant Feature Structural Causal Models (PIIF SCM /FIIF SCM) respectively, to certify their weaknesses in representing fake invariant features, then, unify their causal diagrams to propose ReStructured SCM (RS-SCM). RS-SCM can ideally rebuild the spurious and the fake invariant features simultaneously. \hspace{-0.1em}Given \hspace{-0.1em}this, \hspace{-0.1em}we further develop an approach based on conditional mutual information with respect to RS-SCM, then rigorously rectify the spurious and fake invariant effects. It can be easily implemented by a small feature selection subnet introduced in the IRL family, which is alternatively optimized to achieve our goal. Experiments verified the superiority of our approach to fight against the fake invariant issue across a variety of OOD generalization benchmarks.  

\end{abstract}

\section{Introduction}
A fundamental pressumption of machine learning widely believes that models are trained and tested with samples identically and independently (\emph{i.i.d.}) drawn from a distribution. Whereas in practice, the models are inevitably trained and deployed in ubiquitous scenarios so that they poorly perform than what was expected, due to the violation of the \emph{i.i.d.} condition inducing \emph{distributional shift} across various scenarios. The failure could be understood from the view of representation learning, where the \emph{i.i.d.} condition typically achieves the feature generalizing in a distribution, but unfortunately, at the sacrifice of generalization beyond this observed distribution. It is obviously impossible to obtain the domain universe by collecting data from all scenarios, so how to \emph{\textbf{learn reprsentation}} with limited observed domains for chasing the \emph{\textbf{invariant performance to unseen domains}}, have gradually become the promising trend known as invariant representation learning (IRL) for out-of-distribution (OOD) generalization \cite{shen2021towards,wang2022generalizing}.  

The emergence of IRL dates\hspace{-0.1em} back to approaches for domain adaptation \cite{ganin2016domain,zhao2019learning}
%On Learning Invariant Representations for Domain Adaptation; DANN
where data drawn from a test domain (so-called target domain) can be accessed to quantify the distributional shift, thus, invariant representation is spontaneously obtained while minimizing the domain shift. In the OOD generalization setup, only a few number of domains are available whereas the goal turns to learning the invariant representation to unseen domains. It becomes more challenging since minimizing the observed domain gaps does not imply the model generalization to unseen domains. The recent development of causal inference \cite{peters2016causal,mahajan2021domain} provided a set of innovative principles, \emph{i.e.}, Invariant Causal Prediction (ICP), of connecting the IRL and OOD generalization. 
%and of particular prominent methods would be Invariant Risk Minimization (IRM) \cite{c:21} and several other works derived from the similar spirit \cite{c:21,c:23,c:79}. 
Most IRL frameworks henceforth consider data as an endogenous vairable generated through a Structural Causal Model (SCM) 
\cite{pearl2010causal}, which could be partitioned into different \emph{environment factors} where each one corresponds to \hspace{-0.1em}a specific intervention action taken in the SCM. {In such regards}, IRL aims for the recovery of invariant features via the arbitrary environment interventions for diminishing spurious correlation with the label. Of particular prominent methods are Invariant Risk Minimization (IRM) \cite{arjovsky2019invariant}, Invariant Rationalization (InvRat) \cite{chang2020invariant,li2022invariant}, REx \cite{krueger2021out} and some other approaches in the similar spirit \cite{zhou2022sparse,ahuja2020invariant,li2022invariant}. Their objectives are optimized to prevent the classifier from overfitting to environment-specific properties (Figure.1.({\color{red}a})). 

%and the optimization results in the “optimal invariant predictor” that . 

%Structural Causal Model (SCM)

%The IRL principles behind the methods relies on different generative processes under SCM assumptions consistent with their motivation, then proceed to derive their approaches. 
\vspace{-5pt}\subsection{Fake OOD Invariant Effect}\vspace{-3pt}

Despite the potential
and popularity of IRL, plentiful follow-up studies unveiled IRLs' unreliability to learn invariant representation \cite{kamath2021does,nagarajan2020understanding,rosenfeld2020risks}, in which the most notorious problem is probably the \emph{\textbf{fake invariant effect}} %\footnote{This problem was firstly raised by \cite{rosenfeld2020risks} then \cite{li2022invariant} discussed it as ``pseudo invariant''. But ``pseudo invariant'' refers to non-invariant property and should not be misunderstood in the sense of uncertainty \emph{e.g.}, ``pseudo label'', so we replaced it with ``fake invariant'' instead.}
. Particularly, given each environment factor to identify a specific spurious feature in a SCM, if the number of latent environment factors less than the capacity
of spurious features, latent spurious correlation would pretend as an invarant part of the algorithm-recovered features recovered by IRL (Figure.1.({\color{red}b})). The problem arouses from the existence of underlying shortcut $\Phi(\cdot)$ between invariant causal features $Z_{\rm c}$ and spurious features $Z_{\rm s}$. It receives the spurious variable to endow $[Z_{\rm c},\Phi(Z_{\rm s})]$ with the invariant property across training environments, where the classifier prefers $[Z_{\rm c},\Phi(Z_{\rm s})]$ rather than $Z_{\rm c}$ for IRL. While the OOD generalization easily fails since $Z_{\rm s}$ depends on environments that allows arbitrary change during testing.  

The fake invariance typically rises from the scarcity of environments that implicitly raises the “degree of freedom” of invariant representation. The uncontrolled “degree of freedom” are observed both in the linear and non-linear cases, where several recent efforts attempted to recover the true invariant features through the lens of causality. {However, existing paradigms fail to incorporate $\Phi(Z_{\rm s})$ as a part of SCM. It endangers OOD generalization since no knowledge of the data assumption on the underlying environments may cause a paradox for IRL \cite{ahuja2021invariance}.    
\begin{figure}[t]
	\center
	\includegraphics[width = 0.9\columnwidth]{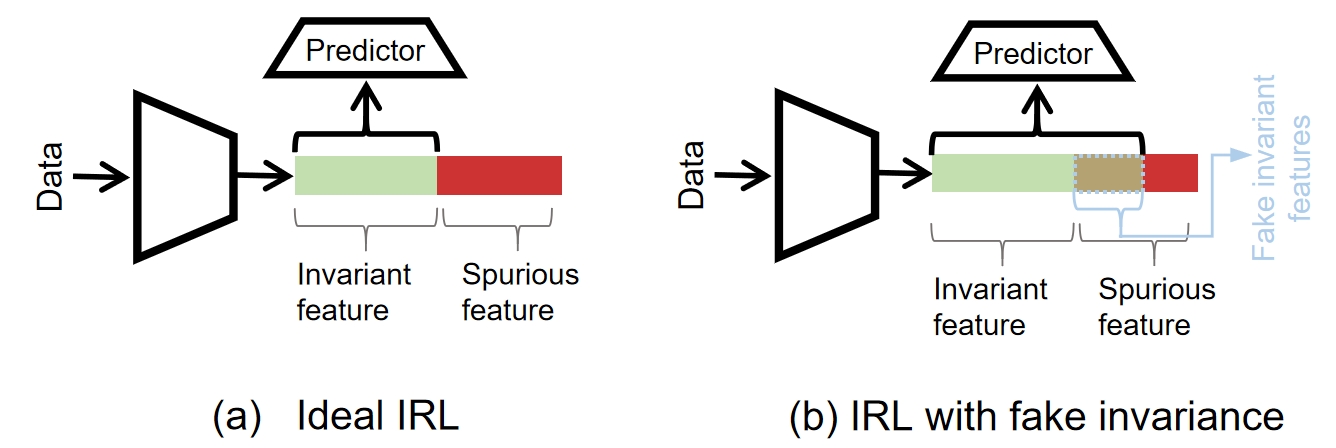}
	\vspace{-6pt}
	\caption{The comparison between ideal IRL and IRL contaminated by fake invariant features (the case in IRM).
	}
	\vspace{-6pt}
	\label{f1}
\end{figure}
\vspace{-5pt}\subsection{Contributions}\vspace{-3pt}

To solve the problem above, our work provides the first rigorous investigation of considering the fake variant shortcut $\Phi(Z_{\rm s})$ as a latent variable across diverse SCM data assumptions.    
Specifically, we firstly 
investigated two famous SCM data assumptions (Partially and Fully Informative Invariant Feature, PIIF SCM and FIIF SCM) commonly employed by existing IRL frameworks \cite{ahuja2021invariance}. \hspace{-0.1em}Under the \hspace{-0.1em}background \hspace{-0.1em}of a \hspace{-0.1em}IRL family derived from \cite{chang2020invariant,li2022invariant}, we certify PIIF SCMs impossibly to incorporate $\Phi(Z_{\rm s})$ whereas the paradigm of FIIF SCM surprisingly suits the information-theoretic properties behind $\Phi(Z_{\rm s})$. To obtain the best of both worlds, we propose a novel ReStructured SCM framework combing PIIF SCM and FIIF SCM to simultaneously rebuild and isolate the spurious and the fake invariant charateristics.    

Given this, we further proved why the IRL family only recovers the label-dependent spurious features but fails to mitigate the fake invariant features under the RS-SCM framework, and propose a conditional mutual information objective to rectify the negative invariant effect caused by $\Phi(Z_{\rm s})$. It can be easily implemented by a subnetwork to select invariant features then merge with the IRL family, which are alternatively trained to prevent invariant representation from the fake invariant effects. Diagnostic experiments and five large-scale real-world benchmarks validates our work.

%new invariant causal prediction framework incorporating the    

%Our work proposes the first 

%{\color{red} \emph{theoretically}. We first consider the cross-modal embedding space following the constant modality-gap phenomenon found by \cite{zhangdiagnosing}. With regards to our empirical observations, we extend the conclusion to learnable prompts to show that more learnable prompts might reduce the constant modality gap more significantly. In terms of constant modality gaps, we further proved the existence of \emph{\textbf{cross-modal unidentifiability issue}}: a paradox confusing the cross-modal model with a single prompt template in visual recognition. It could be restrained by multi-prompting empirically, thus, interpreting why multi-prompt learning could outperform single prompt for the sake of vision-language transferrability.     

%In terms of our retrospect, the main challenge of multi-prompt learning refers to its generalizability. Derived from this concern, we propose a new methodology Energy-based Multi-Prompt Learning (EMPL) for striking the balance between in-domain generalization and open-vocabulary generalization abilities. EMPL implicitly defines an energy-based \cite{lecun2006tutorial} prompt distribution that simultaneously use image and prompt as the variable. \hspace{-0.2em}With this regard, our method could be rigorously treated as . }

\vspace{-2pt}\subsection{Related Work}\vspace{-2pt}

%\vspace{-0pt}\subsection{OOD generalization, IRL, and Causality}\vspace{-0pt}

OOD generalization or domain generalization investigate \hspace{-0.1em}the principles to extend the empirical risk minimization (ERM) to suit the data beyond the training distributions \cite{wang2022generalizing,shen2021towards}. Before IRL becoming the trend, there have been three famous research lines. \emph{Data augmentation} increases the diversity of observed domains by taking complex operations to transform training data, \emph{i.e.}, randomization, mixup, altering location, texture and replicating the size of objects \cite{khirodkar2019domain,wang2020heterogeneous,yu2023distribution}, etc.
%Distribution Shift Inversion for Out-of-Distribution Prediction
\emph{Meta-learning} optimizes a general domain-agnostic model, which turns into a domain-specific version with a few of the domain-specific samples for the test adaptation \cite{shu2021open,chen2023meta}. 
%Improved Test-Time Adaptation for Domain Generalization
Ensemble approaches integrated submodels with regards to diverse training domains to generalize unseen distributions \cite{lee2022cross,chu2022dna}.    
%DNA: Domain Generalization with Diversified Neural Averaging;

Massive OOD generalization literatures are deeply related with IRL. The rationale behind aims to minimize
the upper bounds of the prediction errors in unseen distributions \hspace{-0.1em}which\hspace{-0.1em} used to rely upon the covariate
shift presumptions, yet \cite{chen2021domain,kuang2018stable} implausible when the spurious correlation occurs. Increasing attentions were repayed to causality to tackle the issue. Inspired by ICP \cite{peters2016causal}, a plenty of IRLs treat the predictions as invariant factors across different domains, which recovers the causation from feature to label regardless of environment interventions \cite{arjovsky2019invariant,ahuja2020invariant,chang2020invariant,krueger2021out,li2022invariant,jiang2022invariant}. Some causal learning achieve OOD generalization beyond ICP \cite{jalaldousttransportable,wang2022out,lv2022causality}.  

%Out-of-distribution Generalization with Causal Invariant Transformations

Recent critics are discussed to this roadmap due to the unsatisfied recovery of invariant features. IRMs were denouced since its feasible variant IRMv1 poorly adapts to deep models \cite{zhou2022sparse}
%%(Bayesian Invariant Risk Minimization, Sparse Invariant Risk Minimization) 
and lacks the robustness of environment diversity \cite{huh2022missing,lin2022zin}. \cite{nagarajan2020understanding} uncovered two failure modes caused by geometric and statistical skews in their nature. \cite{rosenfeld2020risks} rigorously exhibited the fake invariance in the linear setting and empirically reported the issue over a wide range of IRL methods. {Despite Invariant Information Bottleneck (IIB) \cite{li2022invariant} claiming their capability to solve this issue, our causal diagnosis inspired by \cite{ahuja2021invariance} verified that their solution powerless of the fake invariance.}        
%%()

Our work is closely related with Pearl's causality, the advanced knowledge \cite{pearl2010causal,pearl2009causal} before reading.

%\vspace{-0pt}\subsection{Difficulties of Causal IRL}\vspace{-0pt}
\begin{figure}[t]
	\center
	\includegraphics[width = 1\columnwidth]{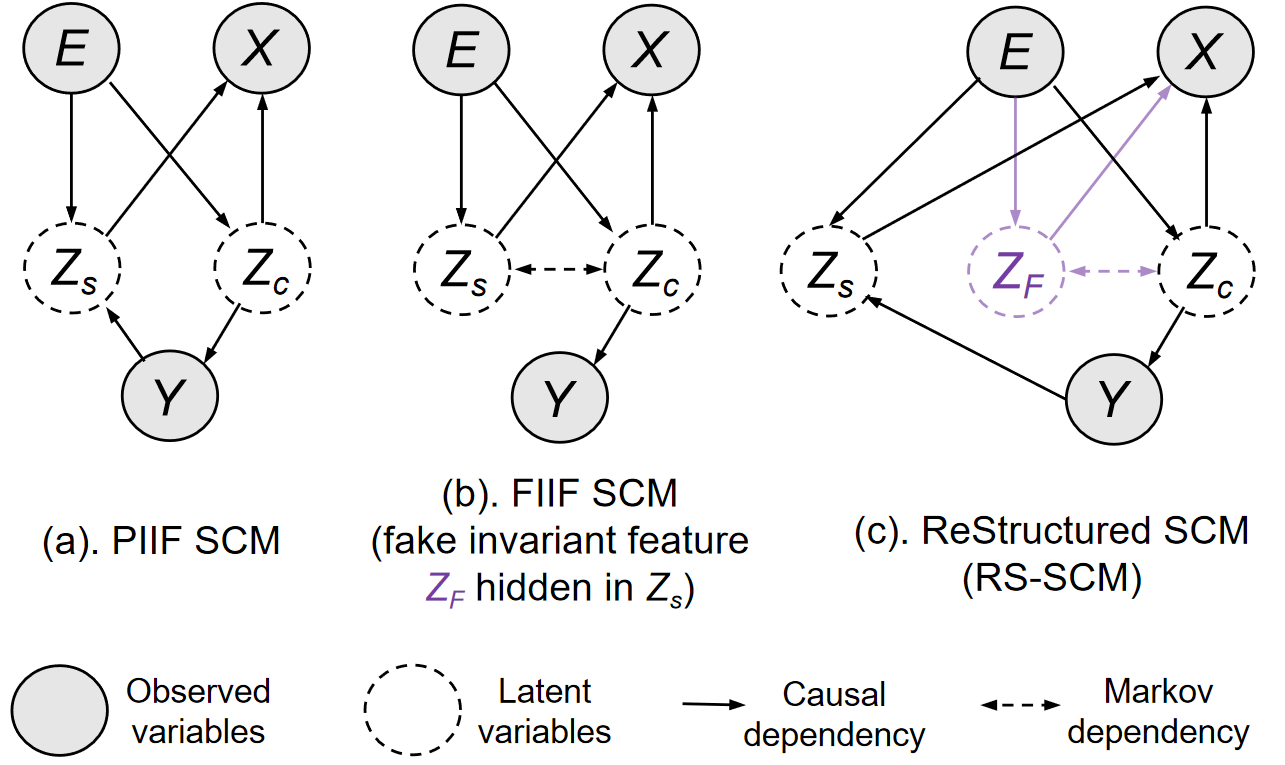}
	\vspace{-6pt}
	\caption{SCMs for InvRat and IIB in distribution shifts:({\color{red}a}). SCM with the PIIF condition ($Y$ $\nperp$ $E$,$Z_{s}$$|$$Z_{c}$); ({\color{red}b}). SCM with the FIIF condition ($Y$ $\perp$ $E$,$Z_{s}$$|$$Z_{c}$) where fake invariant features $Z_{F}$=$\Phi(Z_{s})$ hide; ({\color{red}c}).RS-SCM incorporates the FIIF and PIIF with the fake invariant variable $Z_{F}$.   
	}
	\vspace{-3pt}
	\label{scm}
\end{figure}
\vspace{-2pt}\section{Causal Diagnosis for Fake Invariance}
We first review IRL and its fake invariant issue, then propose a new SCM to reflect the fake invariance in causal diagram.
\vspace{-4pt}\subsection{Preliminary}\vspace{-0pt}
\textbf{OOD Generalization \& IRL setup.} Suppose we are given datasets $\mathcal{D}$$=$$\{\mathcal{D}_{e}\}$ for training, where each one refers to the training environment $e$ $\in$ $\mathcal{E}_{\rm tr}$ collected from the environment universe $\mathcal{E}$ ($\mathcal{E}_{\rm tr}$$\subset$$\mathcal{E}$). For a training set $\mathcal{D}_{e}=\{\boldsymbol{x}_{i}^{e},\boldsymbol{y}_{i}^{e}\}^{n_e}_{i=1}$, each sample with its label was \emph{i.i.d} drawn from the underlying joint distribution $\mathcal{P}_{e}$. The purpose of OOD generalization is to learn a model $f$ with $\mathcal{D}=\{\mathcal{D}_{e}\}_{e\in\mathcal{E}_{\rm tr}}$ for enabling the label prediction to the samples drawn from arbitrary environments in $\mathcal{E}$, thus, minimizing the population risk as follows:      
\begin{equation}\begin{small}
		\begin{aligned}
		\mathcal{R}_{\mathcal{E}}(f)=\max_{e\sim\mathcal{E}} [\mathcal{R}^{e}(f)]=\max_{e\sim\mathcal{E}} \big[\mathbb{E}_{\mathcal{P}_{e}(\boldsymbol{x},\boldsymbol{y})}  \big[\mathcal{L}\big(f(\boldsymbol{x}),\boldsymbol{y}\big)\big]\big]
		,\label{erm}
	\end{aligned}\end{small}
\end{equation}where $\mathcal{L}(\cdot,\cdot)$ denotes the loss function with regards to a task; $f(\boldsymbol{x})$$=$$\rho(h(\boldsymbol{x}))$ in which $h(\boldsymbol{x})$ denotes the feature extracted from $\boldsymbol{x}$ by encoder $h$ and $\rho$ receives $h(\boldsymbol{x})$ to predict the label.

Obviously the population risk in Eq.\ref{erm} is impossible to directly approximate since we have distributions of $\mathcal{E}_{\rm tr}$ instead of $\mathcal{E}$ during training. IRLs resort to training the model $f$ over $\mathcal{D}$ to find the representation with the domain-invariant properties. For instance, InvRat and its derivation IIB \cite{li2022invariant} achieve this goal by maximizing the mutual information (MI) in the usage of the invariance contraint:   
\begin{equation}\begin{small}
	\begin{aligned}
		\max_{\rho,h} \ I[Y;&h(X)] \ \ {\rm \textbf{s.t.}} \ \ Y \perp E \ | \ h(X),
		%-\lambda \ I[Y;E|h(X)],
	\end{aligned}\end{small}\label{InvRat}
\end{equation}where the capital letters denote variables and $\perp$ denotes their probabilistic independence. $I[Y;h(X)]$ denotes the MI between $Y$ (label variable) and $h(X)$ (feature variable) extracted from $X$ (data variable). The invariance contraint equals to minimze the Conditional MI (CMI) $I[Y;E|h(X)]$ between $Y$ and $E$ (environment variable) under the condition $h(X)$. They could be jointly formulated and optimized through their variational upper bounds \cite{alemi2016deep}. 

Existing work mostly investigates IRL from the IRM perspective, while our paper focuses on the IRL paradigms derived from InvRat or IIB. It helps to avoid many optimization issues that conventionally occur in IRM and its variants.

%IRM learns invariant feature to elicit the prediction simultaneously optimal $\forall e\in\mathcal{E}_{\rm tr}$, \emph{i.e.},
%\begin{equation}
%	\begin{aligned}
%		\min_{\rho,h} \ \mathbb{E}_{e\sim\mathcal{E}_{\rm tr}}[\mathcal{R}^{e}&(\rho \circ h)], \\ {\rm \textbf{s.t.}}& \ \rho\in{\arg\min}_{\overline{\rho}} \ \ \mathbb{E}_{e\sim\mathcal{E}_{\rm tr}}[\mathcal{R}^{e}(\overline{\rho} \circ h)].
%	\end{aligned}
%\end{equation}Compared with IRM and its derived approaches, 
 
\textbf{Causal Rationales behind IRLs.} It is known that \textbf{\emph{OOD generalization is possible under practical causal assumptions}}, representing the data generation process with different types of distributional shifts with regards to different environment interventions \cite{ahuja2021invariance}. Resembling the similar principles, we provide a latent-variable SCM perspective to observe how InvRat / IIB generates $X$ from the latent variable $Z$. It is concatenated by the invariant features $Z_{c}$ and the variant counterpart $Z_{s}$ with respect to the environment drawn from $E$. In the regards of latent interaction between $Z_{c}$ and $Z_{s}$, we may further categorize the SCMs into the types \emph{Fully Informative Invariant Features} (FIIF) and \emph{Partially Informative Invariant Features} (PIIF), depending upon whether $Z_{c}$ is fully informed by $Y$, \emph{i.e.}, $Y \perp E,Z_{s} \ | \ Z_{c}$. Formal definitions are provided without additive noise for simplicity:  
\begin{ass}
	[PIIF Structural Causal Model (SCM)]
	\vspace{-3pt}\begin{displaymath}\begin{small}
		\begin{aligned}
			&Y&:=&f_{\rm inv}(Z_{c}), &Z_{c}&:=f_{\rm env}(E), \\&Z_{s}&:=&f_{\rm spu}(E,Y),
			&X&:=f_{\rm gen}(Z_{c},Z_{s});
		\end{aligned}\end{small}
	\end{displaymath}
\end{ass}\vspace{-6pt}
\begin{ass}
	[FIIF Structural Causal Model (SCM)]
	\begin{displaymath}\begin{small}
			\begin{aligned}
			&Y&:=&f_{\rm inv}(Z_{c}), &Z_{c}&:=f_{\rm env}(Z_{s},E), \\&Z_{s}&:=&f_{\rm spu}(Z_{c},E),
			&X&:=f_{\rm gen}(Z_{c},Z_{s}).
		\end{aligned}\end{small}
	\end{displaymath}
\end{ass}
%It ganrantees the OOD generalization (\emph{i.e.}, Eq.\ref{erm}) to the IRM variants with Information Bottleneck (IB) constraint. 
%In visual realisms, $Z_{c}$ may denote the label-relevant shape and texture, $Z_{s}$ may denote label-inrelevant background and color information, then $E$ may indicate the style and locations. 

Figure.\ref{scm}.\hspace{-0.02em}(a-b) \hspace{-0.02em}visualize the causal diagrams under the assumptions for InvRat and IIB. Despite representing different classes of distributional shifts, PIIF and FIIF SCMs simultaneously consider $Z_{s}$ as the spurious correlation of the causal routine from data $X$ to label $Y$, then, seeking to deconfound the spurious factors from $E$. Their common goal is to learn the feature encoder $h(\cdot)$ for recovering $Z_{c}$, then facilitate the invarint causal prediction $Z_{c}$$\overset{f_{\rm inv}}{\rightarrow}$$Y$. It is highlighted that the previous work discussed for InvRat /IIB were almost derived from the PIIF Assumption (Figure.\ref{scm}.({\color{red}a})). However, the FIIF Assumption (Figure.\ref{scm}.({\color{red}b})) was seldom investigated because spurious correlations mostly live in the situations where $Z_{s}$ is partially informed by $Y$ ($Y$ $\not\perp$ $E$,$Z_{s}$$|$$Z_{c}$), \emph{e.g.}, the spurious features refer to visual background information in an image. In contrast, our work prefers to the necessity of incorporating the FIIF SCM Assumption since 
	\emph{{the fake invariant effects are probably hidden in the FIIF-SCM Assumption}}.

%The aforementioned approaches could be proved to ganrantee the causal recovery of the invariant part $Z_{c}$ under their SCM assumptions. 
%Besides, we could further categorize $Z_{c}$ into \emph{Fully Informative Invariant Features} (FIIF) or \emph{Partially Informative Invariant Features} (PIIF), depending on whether the invariant part $Z_{c}$ is fully informed by $Y$, \emph{i.e.}, $Y \perp E,Z_{s} \ | \ Z_{c}$. It ganrantees the OOD generalization (\emph{i.e.}, Eq.\ref{erm}) to the IRM variants with Information Bottleneck (IB) constraint.    
%Apart from this, we also observe 

\vspace{-2pt}\subsection{Fake Invariance from A Causal Lens}\vspace{-1pt}
To verfiy our claim, we need to overview the concept of fake invarian features. As demonstrated by \cite{li2022invariant}, they are some spurious features $Z_{s}$ that pretend to be the domain-invariant part of $Z$ by the shortcut $\Phi(\cdot)$ from two viewpoints 
\begin{itemize}
		\item $Y$$\perp$$E$$|$ $Z_{c}$,$\Phi(Z_{s})$ (\textbf{Fake invariance}): Combining $Z_{c}$ and $\Phi(Z_{s})$ can produce the lower empirical risk than the invariant risk, implying the domain-invariant property.
	\item $Y$$\not\perp$$E$$|$ $\Phi(Z_{s})$ (\textbf{Spuriousness}): For arbitrary $Z_{s}$,$Y$,$E$ under the SCM Assumptions aforementioned, we discover $Y$$\not\perp$$E$$|$ $Z_{s}$$\leftrightarrow$$I[Y;E|Z_{s}]$$>$$0$. Consider the shortcut $\Phi()$ effect on $I[Y;E|Z_{s}]$ reduced to 
	\begin{displaymath}\vspace{-4pt}\begin{small}
			\begin{aligned}
				&I[Y;E|\Phi(Z_{s})]\\=&I[Y;E]-\big(I[Y;\Phi(Z_{s})]- I[Y;\Phi(Z_{s})|E]\big)
				\\\geq&I[Y;E]-\big(I[Y;Z_{s}]- I[Y;Z_{s}|E]\big)
				\\=&I[Y;E|Z_{s}]>0,
			\end{aligned}
		\end{small}
	\end{displaymath}which leads to such property (The derivation details refer to our Appendix).
%\footnote{The derivation comes from Data Processing Inequality.}	
\end{itemize}\vspace{-3pt}Given these, we reconsider the environment interventions in the PIIF and FIIF SCM Assumptions, respectively, then discuss whether $\Phi(Z_{s})$ embedded in their frameworks. Notice that $\Phi(\cdot)$ is solely the pathway algorithmic-recovered from the feature encoder $h(\cdot)$. It does not refer to any dependency under SCM assumptions.

\textbf{PIIF SCM Fail to \hspace{-0.1em}Identify the \hspace{-0.1em}Fake \hspace{-0.1em}Invariance}.\hspace{-0.1em} In the PIIF SCM Assumption, \hspace{-0.1em}some evidences in Figure.\ref{scm}({\color{red}b}) can be readouted through the $d$-seperation rules \cite{pearl2010causal}: 
\begin{itemize}
	\item $Y$$\not\perp$ $E$ implies the label marginal alters according to the environment interventions (the \emph{non i.i.d.} property);
	\item $Y$$\perp$ $E$ $|$$Z_{c}$ demonstrates the independence between the label and the environment intervention provided with $Z_{c}$. So $h(\cdot)$ recovers the invariant features from $Z_{c}$ and may achieve OOD generalization regardless of $E$;
	\item $Y$$\not\perp$ $E$ $|$$Z_{s}$ and $Y$$\not\perp$ $E$$|$$Z_{c}$,$Z_{s}$ respectively demonstrate that (1). $Z_{s}$ indicates the non-invariant property that we mentioned; (2). If $h(\cdot)$ recovers $Z_{c}$, $Z_{s}$ simultaneously, the independence between $Z_{c}$ and $Y$ will not hold since $Z_{s}$ works for a collider between $E$ and $Y$.  
\end{itemize}
The observations explain a broad range of questions of IRLs except for the fake invariant phenomenon. Specifically, such issue is typically caused by the spurious features $Z_{s}$ via the shortcut $\Phi(Z_{s})$, whereas $\Phi(Z_{s})$ can not be reflected by $Z_{s}$ in the PIIF SCM since evidences $Y$$\not\perp$ $E$ $|$$Z_{s}$ and $Y$$\not\perp$ $E$$|$$Z_{c}$,$Z_{s}$ prevent the encoder $h(\cdot)$ from recovering features in $Z_{s}$. But what if $\Phi(Z_{s})$ indicates a part of $Z_{c}$? This conjecture is also impossible in the PIIF SCM setup due to the \emph{spuriousness} of $\Phi(Z_{s})$: $Y$$\not\perp$$E$$|$ $\Phi(Z_{s})$ obviously conflicted with the domain-invariant property required for the features in $Z_{c}$. 

%But what if the fake invariance becomes a part of $Z_{c}$ under the PIIF SCM? It sounds reasonable since $Z_{c}$ performs the invariant property that exactly suits the fake invariant charateristic while training, yet it requires us to protect the feature recovery from $\Phi(Z_{s})$ hidden in $Z_{c}$.   

\textbf{FIIF SCM Implies the Fake Invariance}. 
%The analysis for the PIIF SCM Assumption demonstrates that whether $Z_{s}$ or $Z_{c}$, they both do not suit to denote fake invariant features by the causal diagram in Figure.\ref{scm}({\color{red}a}). 
Provided with the failure witnessed in the PIIF SCM Assumption, we turn to the FIIF SCM Assumption and verify why the fake invariant features could be distinctly denoted as $Z_{s}$ in Figure.\ref{scm}({\color{red}b}). We analyze the causal independences behind the FIIF SCM Assumption by following the same routine of the PIIF SCM Assumption. Such data distribution satisfies
\begin{itemize}
	\item $Y$$\not\perp$ $E$ and $Y$$\perp$ $E$ $|$$Z_{c}$ both hold as the FIIF SCM does;
	\item $Y$$\not\perp$ $E$ $|$$Z_{s}$ indicates $Z_{s}$'s the spurious nature; 
	\item $Y$$\perp$ $E$ $|$$Z_{c}$,$Z_{s}$ demonstrates that combing $Z_{c}$ and $Z_{s}$ leads to the invariant representation, however, $Z_{s}$ should not be included since it implies spurisous correlations.
\end{itemize}Observe that the second property suggests $Z_{s}$ being domain-specific for the label prediction, whereas combining $Z_{c}$ and $Z_{s}$ \hspace{-0.1em}results in\hspace{-0.1em} the \hspace{-0.1em}domain\hspace{-0.1em}-\hspace{-0.1em}invariant property that intervents the causal prediction over $Z_{c}$,$Z_{s}$$\overset{f_{\rm inv}}{\rightarrow}$$Y$, which should have been $Z_{c}$$\overset{f_{\rm inv}}{\rightarrow}$$Y$ instead. In this case, the second and the third observations for $Z_{s}$ in the FIIF SCM setup do exactly refer to the \emph{fake invariance} and the \emph{spuriousness} behind $\Phi(Z_{s})$.

\textbf{ReStructued SCM}. Despite incorporating the fake invariance, $Z_{s}$ in the FIIF SCM Assumption contradicts the PIIF-SCM spurious correlation commonly found in practice. To obtain the best of both worlds, we restructure their SCMs and propose a new data generation regime that unify the PIIF and FIIF spurious correlations:   
\vspace{-2pt}\begin{ass}
	[ReStructured SCM (RS-SCM, Figure.\ref{scm}({\color{red}c}))]
	\begin{displaymath}\begin{small}\begin{aligned}
			Y:=&f_{\rm inv}(Z_{c}), &Z_{c}&:=f_{\rm env}(E, Z_{F}), \ \ Z_{s}:=f_{\rm spu}(E,Y),\\Z_{F}:=&f_{\rm fake}(E,Z_{c}),
			&X&:=f_{\rm gen}(Z_{c},Z_{s});
			\end{aligned}\end{small}
	\vspace{2pt}
	\end{displaymath}
\end{ass}\vspace{-1pt}The RS-SCM extends the previous PIIF SCM by branching the $Z_{s}$-based spurious features to embrace the $Z_{F}$ as our fake invariant features \emph{i.e.}, $Z_{F}$$=$$\Phi(Z_{s})$. The independencies between $Z_{F}$ and the other variables keep consistent with the spurious features $Z_{s}$ used in the FIIF SCM Assumption. Notably, the fake invariant effect only happens while the training environments overloaded with all spurious factors, therefore the causal subgraph with respect to $Z_{F}$ in the RS-SCM should be adaptively deactivated beyond this situation. The switchable SCM mechanism is inspired from the heterogeneous causal graph \cite{watson2023heterogeneous}%%
, where we highlight the switchable parts by purple in Figure.\ref{scm}({\color{red}c}).    
\begin{rem}
	The RS-SCM Assumption concurrently embeds spurious features and fake invariant features. 
\end{rem}
\vspace{-6pt}\section{Methodology}\vspace{-1pt}
In this section, we elaborate our methodology derived from the restructured causality. We first review the strategies in InvRat and IIB, then, showing how they fail to rectify $\Phi(Z_{s})$. Then we formulate our rectification objective to calibrate InvRat and IIB in an invariant learning manner. 
\vspace{-5pt}\subsection{InvRat Family Does Not Rectify $\Phi(Z_{s})$}\vspace{-3pt}
In the InvRat family, the vanilla InvRat obviously fails due to no effort paid to rectify $\Phi(Z_{s})$ by Eq.\ref{InvRat}. Its derivation IIB advocates the minmal information between $X$ and $h(X)$, \emph{i.e.}, $\min_{h} I[X;h(X)]$. The constraint pernalizes the capacity of invariant feature recovery in $h(X)$, then combined with the invariant constraint $\min_{h}I[Y;E|h(X)]$. It was deemed to remove $\Phi(Z_{s})$ hidden in the recovered feature $h(X)$, which is unreliable since their analysis is built upon the PIIP SCM Assumption where $\Phi(Z_{s})$ can not be reflected by their latent variables. But under our RS-SCM Assumption, whether the constraint $\min_{h} I[X;h(X)]$ enables the fake invariance elimination? Our theoretic result also denies such guess: 
\begin{prop}\label{prop1}
	In the RS-SCM Assumption, given invariant feature $Z_{c}$$\in$$\mathbb{R}^{n_c}$ and fake invariant features $\Phi(Z_{s})$$\in$$\mathbb{R}^{n_F}$ as a feature subset of fake invariant variable $Z_{F}$: $\Phi(Z_{s})$$\subset$$Z_{F}$, 
	%Given $Z_1$,$Z_2$ as the feature subsets of $Z_{c}$, $\Phi(Z_{s})$ (\emph{i.e.},$Z_1$$\subset$$\mathbb{R}^{n_1}$, $Z_2$$\in$$\mathbb{R}^{n_2}$, $Z_2$$\in$$\mathbb{R}^{n_2}$ where $n_1$$<$$n_c$ and $n_2$$<$$n_F$), we have $\min_{Z_1,Z_2}I[Y;E|Z_1,Z_2]$ $=$ $0$ ($E$$\in$$\mathcal{E}_{\rm tr}$). 
	%So given arbitrary $n'$$\leq$$n_c$, 
	we can find $Z$$\in$$\mathbb{R}^{n_c}$ with $Z$$\cap$$\Phi(Z_{s})$$\neq$$\emptyset$ that satisfies
	\begin{equation}\begin{aligned}
			\lambda I[Y;E|Z]+\beta I[X;Z]&\\ \leq \lambda I[Y;E|Z_{c}\cup\Phi(Z_{s})]+&\beta I[X;Z_{c}\cup\Phi(Z_{s})],\\ &\ \ \ \ \ \ \ \ \ \ \ \textbf{s.t.} \ \ \forall \lambda,\beta\in\mathbb{R^{+}}.
		\end{aligned}
	\end{equation}
\end{prop}
%But unfortunately, this strategy is of no avail since $\min_{h} I[X;h(X)]$ fail to identify the fake invariance while training:

%Notwithstanding the fake invariance $Z_{F}$ isolated from $Z_{s}$ and $Z_{c}$, the IIB strategy remains unable to eliminate $\Phi(Z_{s})$ on account of the :

%It would be a controversy of embedding the fake invariance into the causal diagram. In paticular, according to the

\noindent The justification elaborates that for each invariant feature $Z_{c}$ recovered by $h()$, the FIIP SCM may search the feature $Z$ in the identical latent space of $Z_{c}$ to bound the invariant contraint of the IIB strategy, however, $Z$ satisfies $Z$$\cap$$\Phi(Z_{s})$$\neq$$\emptyset$ wherein the fake invariant features might be included by this reprsentation. IRLs conventionally optimize their models by way of non-convex variational bounds, thus  hardly to certify whether training $h(\cdot)$ may result in the recovery of $Z$ or $Z_c$. In terms of Proposition.\ref{prop1} and what we previously discussed, the conlusion is drawn to the InvRat family in RS-SCM:   
\vspace{-1pt}\begin{rem}\label{rem2}
	Under the RS-SCM Assumption, InvRat and IIB strategies distinguish the spurious features $Z_{s}$ whereas fail to elimiate the fake invariant features $Z_{F}=\Phi(Z_{s})$.
\end{rem}
\noindent Remark.\ref{rem2} illustrates the bright side of the InvRat familiy: the ability to debias $Z_{c}$,$Z_{F}$$\overset{f_{\rm inv}}{\rightarrow}$$Y$ from the spurious factors $Z_{s}$ by the invariant independence contraint $\big($\emph{i.e.}, $Y$$\perp$ $E$ $|$$h(X)$$\big)$. To achieve OOD generalization, we are required to prevent existing InvRat variants from the unexpected recovery of $Z_{F}$. 

%\noindent It demands another proposal to negate the influence of $Z_{F}$.   

\vspace{-3pt}\subsection{Rectification Approach by RS-SCM}\vspace{-0pt}
We move foward our discussion of how to wipe out $Z_{F}$ from $Z_{c}$$\cup$$Z_{F}$. Under the RS-SCM Assumption, we reconsider the Markov dependency across $Z_{c}$ and $Z_{F}$ then distinguish them according to their different behaviors for the label prediction conditioned with each other and $Z_{s}$. Specifically, when provided with $Z_{F}$ and $Z_{s}$, the $d$-seperation principle judges the causal prediction $Z_{c}$$\overset{f_{\rm inv}}{\rightarrow}$$Y$ with $Y$$\not\perp$$Z_{c}$$|$ $Z_{F}$,$Z_{s}$. It implies the causal path activated to maximize the CMI:
%however, this label prediction is also confounded by $E$ because the chosen $Z_{s}$ turns into a collider in $E$$\overset{f_{\rm spu}}{\rightarrow}$$Z_{s}$$\overset{f_{\rm spu}}{\leftarrow}$$Y$. In this consideration, $Y$$\not\perp$$Z_{c}$$|$ $Z_{F}$,$Z_{s}$ might not be a good constraint to identify $Z_{c}$. 
\begin{equation}\label{CMImax}
	\begin{small}\begin{aligned}
			\max_{Z_{c}} I[Y;Z_{c}|Z_{F},Z_{s}]=\max_{Z_{c}} I[Y;Z_{c}|h(X)/Z_{c}],
		\end{aligned}
	\end{small}
\end{equation}which is optimized for recovering $Z_{c}$ from invariant encoder $h(X)$ learned by InvRat or IIB. 

Similarly, we observe the causal dependency across $Y$ and $Z_{F}$, then analyze the label prediction $Z_{F}$$\overset{f_{\rm env}}{\rightarrow}$$Z_{c}$$\overset{f_{\rm inv}}{\rightarrow}$$Y$ given $Z_{s}$,$Z_{c}$ as the condition. It refers to $Y$$\perp$$Z_{F}$$|$ $Z_{c}$,$Z_{s}$ that equivalently minimizes the CMI constraint:
\begin{equation}\label{CMImin}
	\begin{small}\begin{aligned}
		\min_{Z_{F}} I[Y;Z_{F}|Z_{c},Z_{s}]=\min_{Z_{F}} I[Y;Z_{F}|h(X)/Z_{F}].		
		\end{aligned}
		\end{small}
\end{equation}The contraint above helps us to identify $Z_{F}$ from $h(X)$. 

Note that when $h(X)$ has been well trained by the InvRat or the IIB, their models encourage $h^\ast(X)$$=$$Z_{F}$$\cup$$Z_{c}$. The nice property unifies Eq.\ref{CMImax} and Eq.\ref{CMImin} into the same objective, \emph{i.e.},     
\begin{equation}\label{md}\begin{small}
	\begin{aligned}
		&\min_{Z_{c},Z_{F}} I[Y;Z_{F}|h^\ast(X)/Z_{F}]-\lambda I[Y;Z_{c}|h^\ast(X)/Z_{c}]\\
		=&\min_{Z_{c}} I[Y;h^\ast(X)/Z_{c}|Z_{c}]-\lambda I[Y;Z_{c}|h^\ast(X)/Z_{c}].
	\end{aligned}
\end{small}
\end{equation}where $\lambda$ indicates the trade-off co-efficient. Maximizing and minimizing CMI are intractable while thanks to the symmetry between \hspace{-0.1em}Eq.\ref{CMImax} \hspace{-0.1em}and \hspace{-0.1em}Eq.\ref{CMImin}, \hspace{-0.1em}the CMI decomposition holds as        
\begin{displaymath}\begin{small}
		\begin{aligned}
		 &I[Y;h^\ast(X)/Z_{c}|Z_{c}]=-H(Y|h^\ast(X))+H(Y|Z_{c});\\
		 &I[Y;Z_{c}|h^\ast(X)/Z_{c}]=-H(Y|h^\ast(X))+H(Y|h^\ast(X)/Z_{c}),
		\end{aligned}
	\end{small}\end{displaymath}so that we simplify Eq.\ref{md} for the $Z_{c}$ recovery from $h^\ast(X)$:
\begin{equation}\label{md2}\begin{small}
		\begin{aligned}
			&\min_{Z_{c}} \ I[Y;h^\ast(X)/Z_{c}|Z_{c}]-\lambda I[Y;Z_{c}|h^\ast(X)/Z_{c}]\\
		=&\min_{Z_{c}}	\ H(Y|Z_{c})-\lambda H(Y|h^\ast(X)/Z_{c}) + (\lambda-1)H(Y|h^\ast(X))
	\end{aligned}
	\end{small}
\end{equation}where $(\lambda-1)H(Y|h^\ast(X))$ is constant in the optimization. The objective implies that rectification only needs to select features from $h^\ast(X)$ to improve the causal invariant prediction $Z_{c}$$\overset{f_{\rm inv}}{\rightarrow}$$Y$ (the first term) and discourage the fake invariant prediction $Z_{F}$$\overset{f_{\rm env}}{\rightarrow}$$Z_{c}$$\overset{f_{\rm inv}}{\rightarrow}$$Y$(the second term). 
The joint CMI nature behind Eq.\ref{md2} holds the theoretical ganrantee as
% the fake invariant features simultaneously identified and discarded:    
\begin{prop}
	Suppose that $h^\ast(X)$$=$$Z_{F}$$\cup$$Z_{c}$ under the RS-SCM Assumption. If feature $Z$ recovered from $h^\ast(X)$ satisfies $I[Y;h^\ast(X)/Z|Z]$$=$$0$ and $I[Y;Z|h^\ast(X)/Z]$$>$$0$, it holds $Z=Z_{c}$ or $Z=Z_c\cup Z_F$. 
	%so that the fake invariant features in $Z_{F}$ are rectified.
\end{prop}The proposition implies the rectification may lead to the ideal $Z$$=$$Z_{c}$ or the trivial result that collapses into $Z_c\cup Z_F$. To prevent the trivial solution, we encourage the joint CMI objective optimized along with $\max_{Z}I(Z;h*(X)/Z)$, where $I(Z;h*(X)/Z)$$>$$0$ helps to get rid of the collapse (More refers to Appendix.A). 

\vspace{-3pt}\subsection{Interplay Invariant Learning}\vspace{-1pt}
Given the rectification approach by Eq.\ref{md2}, we propose a novel framework for OOD generalization.

%Distinct from the InvRat family, Eq.\ref{md2} selects features independent with $X$, thus, the outer-loop phase performs as meta-learning find    
\begin{figure}[t]
	\center
	\includegraphics[width = 0.9\columnwidth]{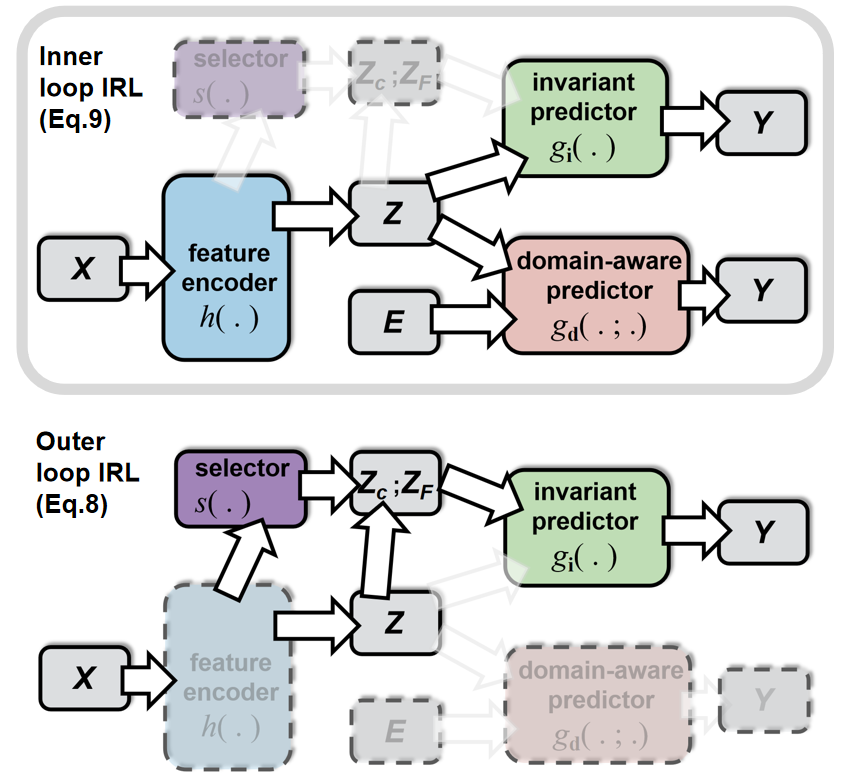}
	\vspace{-6pt}
	\caption{Our interplay invariant learning (IIL) framework. The transparency implies the network frozen or the variables not activated in this alternative phase (Best viewed in color).   
	}
	\vspace{-4pt}
	\label{framework}
\end{figure}
\vspace{-0pt}
\textbf{Neural Soft-Feature Selector.} Eq.\ref{md2} demands a small network that selects \hspace{-0.1em}features \hspace{-0.1em}from $h(X)$ \hspace{-0.1em}to recover\hspace{-0.1em} $Z_{c}$. \hspace{-0.1em}We take a simple two-layer architecture then adjust the scale complexity according to the tasks involved. The sub-network $s(\cdot)$ receives the latent layer's output from $h(X)$ to make the soft feature selection on $h(X)$. Specifically, $s(h(X))$ goes through a series of sigmoid activation functions to yield a vector with the same dimension of $h(X)$, where each positive output means that the feature is selected as $Z_{c}$. So $s(h(X))\odot$$h(X)$ corresponds to the soft-feature selection for $Z_{c}$ and $\big(\mathbf{1}-s(h(X))\big)\odot$$h(X)$ corresponds to the soft-feature selection for $Z_{F}$ ($\odot$ indicates the entry-wise product). The objective Eq.\ref{md2} turns into        
\begin{equation}\label{md3}\begin{small}
		\begin{aligned}
		\min_{s}	\ \mathbb{E}\big[\mathcal{L}\big(Y,s(h(X))&\odot h(X)\big)\big]\\-&\lambda \mathbb{E}\big[\mathcal{L}\big(Y,(\mathbf{1}-s(h(X)))\odot h(X)\big)\big] 
		\end{aligned}
	\end{small}
\end{equation}where we take the task-specific loss to approximate the conditional entropy,\emph{i.e.}, $\mathbb{E}[\mathcal{L}(Y,Z)]$$\rightarrow$$H(Y|Z)$.

\indent\textbf{Framework.} We show how to combine Eq.\ref{md3} with InvRat to learn invariant representation alternatively. Derived from the variational upper bounds of Eq.\ref{InvRat}, the InvRat family plays an adversarial game to jointly train three sub-networks, \emph{i.e.}, feature encoder $h(\cdot)$, invariant predictor $g_{\rm i}(\cdot)$, domain-aware predictor $g_{\rm d}(\cdot)$:    
\begin{equation}\label{invrat2}
	\begin{small}
		\begin{aligned}
			\min_{h,g_{\rm i}}\max_{g_{\rm d}} \ &\mathbb{E}\big[\mathcal{L}(Y; g_{\rm i}(h(X)))\big]\\&+\beta\big(\mathbb{E}\big[\mathcal{L}(Y; g_{\rm i}(h(X)))\big]-\mathbb{E}\big[\mathcal{L}(Y; g_{\rm d}(h(X)))\big]\big)
		\end{aligned}
	\end{small}
\end{equation}Given subnetworks pre-trained by Eq.\ref{invrat2}, our invariant learning framework alternatively performs to (1). train the feature selector $s(\cdot)$ with respect to Eq.\ref{md3} in the outer loop; (2). fine-tune the InvRat subnetworks by Eq.\ref{invrat2} in the inner loop. It is illustrated in Figure.\ref{framework}.

\vspace{-3pt}\section{Experiments}\vspace{-0pt}
In this section, we firstly conduct the diagnostic experiments on the benchmarks derived from recent studies \cite{arjovsky2019invariant,ahmed2020systematic} broadly applied in IRL for OOD generalization. It aims to validate whether our IIL can (1). rectify the fake invariant effect as demonstraed by our theoretical analysis; (2). remain the capability of InvRat and IIB to learn invariant representation. Afterwards, we evaluate our IIL framework in five competitive benchmarks for domain generalization in the wild, in order to verify IIL's feasibility in complex scenarios. Notice that InvRat is originally proposed for rationalization but IIB exactly share the most of its optimization pipelines beyond the MI constraint $I[X;h(X)]$. In this regards, our experiments consider InvRat as the IIB without this regularization.   

\textbf{Benchmarks.} The diagnostic study provides the forensic of IRL baselines under the RS-SCM Assumption. It requires the datasets generated by the same causal mechanism, however, existing diagnostic benchmarks are generated by either FIIF or PIIF SCM, hardly fullfiling our demand \cite{arjovsky2019invariant,ahmed2020systematic}. We observe that RS-SCM consists of FIIF and PIIF SCM Assumptions so that combine their generation recipes to build our diagnostic benchmark to evaluate the invariant learning quality. 
\begin{table}[ht]
	\centering
	\setlength{\tabcolsep}{2.5pt}
	\vspace{-6pt}{\fontsize{9}{15}\selectfont
		\begin{tabular}{c c c c c c  c c c c c  }
			\toprule[1.1pt]
			Dataset &$Z_{c}$ &$Z_{s}$ & $Z_{F}$ &Training/Test Samples \cr
			\midrule[1pt]
			CS-MNIST-CIFAR&Digit&CIFAR&Color&\begin{minipage}[b]{0.3\columnwidth}
				\centering
				\raisebox{-.45\height}{\includegraphics[width=0.5\linewidth]{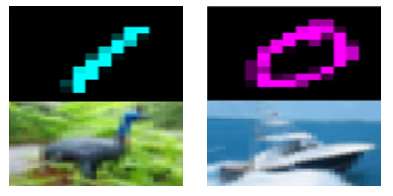}}
			\end{minipage}\cr\hline
			CS-MNIST-COCO&Digit&Object&Color&\begin{minipage}[b]{0.3\columnwidth}
				\centering
				\raisebox{-.45\height}{\includegraphics[width=0.5\linewidth]{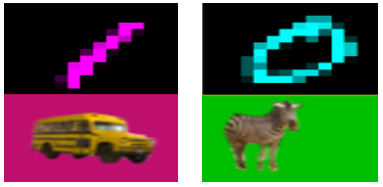}}
			\end{minipage}\cr
			\bottomrule[1pt]
	\end{tabular}}\vspace{-5pt}\caption{The summary of our diagnostic benchmarks.}\label{t1}\vspace{-12pt}
\end{table} 

Specifically, we consider the ten-class digit classification mission derived from the CS-MNIST (FIIF) benchmark in \cite{ahuja2021invariance}. It consists of two environments for training with 20,000 samples each and one environment for evaluation with the same number of data. In this setup, $Z_c$ would refer to the shape of digit and ten digit classes would be associated with ten colors respectively, with an environment-specific probability $p_{e}$. The color indicates the fake invariant variable $Z_{F}$ and the association indicates the Markov dependency between $Z_c$ and $Z_{F}$, and $p_{e}$ implies the environment dependences from $E$ to latent variables $Z_F$ and $Z_c$: $p_{e}$ indicates their association activated otherwise the digit would randomly associated with the ten colors.          
 
The generated digits can not represent the RS-SCM since the spurious factor $Z_s$ has not been included. In this case, we resemble the composition rule in CIFAR-MNIST \cite{zhou2022sparse} whereas we classify MNIST instead of CIFAR. So given each colored digit generated by the previous strategy, we combine it with a CIFAR image drawn from the generative process following the PIIF SCM Assumption \cite{arjovsky2019invariant}. Specifically, given a digit generated by the previous process, we take a random flip with 25\% chance to randomly change its label; then we associate this digit-class label with a CIFAR class with environment-dependent probability $\hat{p}_{e}$. So we have the CS-MNIST-CIFAR to represent the RS-SCM Assumption where the spurious vairable $Z_{s}$ is indicated by the CIFAR classes. We replay this process with Color-COCO then get the second RS-SCM benchmark CS-MNIST-COCO (see Table.1). $p_{e}=1,0.9$ and $\hat{p}_{e}=1,0.9$ are set up for two training environments, respectively. 

Beyond the diagnostic datasets, we also conduct the experiments on VLCS, PACS, Office-HOME, Terra-Incognita and DomainNet, which refer to DomainBed \cite{gulrajani2020search} for real-world OOD generalization.          
 
\textbf{Experimental setup.} In terms of the feature encoder, invariant predictor and domain-aware predictor, we employs the architecutres applied in \cite{li2022invariant}. We take a simple two-layer network for CS-MNIST and a transformer-like subnetwork for DomainBed as our neural feature selectors.  
%They both receive the second last output from the feature encoder to produce a soft binary vector to reweight the feature encoder's output. The hyperparamter and optimization setups for pre-training the InvRat and IIB exactly follows the original work. In the alternative phase, we employs the previous optimization setup both for the inner loop and outer loop, then after an outer loop epoch finishes, IIL alters to the inner loop to fine-tune the subnetworks with $10$ iterations until the in-distribution performance converges. 

%\vspace{-4pt}\subsection{Diagnostic Experiment}\vspace{-2pt} 
\vspace{-4pt}\subsection{Diagnostic OOD Generalzation}\vspace{-2pt}
\begin{table}[h]
	\centering
	\setlength{\tabcolsep}{2.5pt}
	\vspace{-5pt}{\fontsize{8.5}{11.5}\selectfont
		\begin{tabular}{c c c c c c  c c ccc  }
			\toprule[1.5pt]
			\multicolumn{1}{c}{Methods}&\multicolumn{2}{c}{{ID Acc ($\uparrow$)}}&\multicolumn{4}{c}{OOD Generalization Acc ($\uparrow$)}\cr\cline{2-2}\cline{4-6}
			&No shift&& $Z/Z_s$&$Z/Z_F$&$Z/(Z_s,Z_F)$ \cr
			ERM &93.22&&11.87&13.76&10.13 \\
			IRM &94.49&&59.58&53.43&50.06 \\
			IRM+IB &96.14 &&66.98&70.42&59.71 \\\hline
			InRav &89.25  &&63.39 &{65.75} &60.89   \\
			IIB &\textbf{96.76} &&\textbf{70.98}  &{69.42} &66.23\\ 
			InvRav(+ours) &91.72 $\uparrow$ && 64.47$\uparrow$ &{69.23} $\uparrow$ &64.40 $\uparrow$\\
			IIB(+ours) &95.17$\downarrow$ &&70.19$\downarrow$ &\textbf{70.54} $\uparrow$ &\textbf{68.37} $\uparrow$\\
			\bottomrule[1.5pt]
	\end{tabular}}\vspace{5pt}\caption{ID / OOD generalization accuracies on CS-MNIST-CIFAR. $Z/Z_s$, $Z/Z_c$, and $Z/(Z_s,Z_c)$ indicate different distributional shifts between training and test (without spurious factor, without fake invariant, without the both).}\label{t2}\vspace{-6pt}
\end{table}
\begin{table}[h]
	\centering
	\setlength{\tabcolsep}{2.5pt}
	\vspace{-5pt}{\fontsize{8.5}{11.5}\selectfont
		\begin{tabular}{c c c c c c  c c ccc  }
			\toprule[1.5pt]
			\multicolumn{1}{c}{Methods}&\multicolumn{2}{c}{{ID Acc ($\uparrow$)}}&\multicolumn{4}{c}{OOD Generalization Acc ($\uparrow$)}\cr\cline{2-2}\cline{4-6}
			&No shift&& $Z/Z_s$&$Z/Z_F$&$Z/(Z_s,Z_F)$ \cr
			ERM&92.63&&10.24&11.47&9.67 \\
			IRM &94.49&  &49.67&54.26 &47.19 \\
			IRM+IB &\textbf{96.14}&  &56.91 &63.92  &{55.27} \\\hline
			InRav &89.25  &&53.07 &{61.75} &51.33   \\
			IIB &92.44 &&61.14  &{66.38} &57.62\\ 
			InvRav(+ours) &91.72 $\uparrow$ &&58.86 $\uparrow$ &{66.42} $\uparrow$ &{56.16} $\uparrow$\\
			IIB(+ours) &93.17 $\uparrow$ &&\textbf{63.72} $\uparrow$ &\textbf{69.59} $\uparrow$ &\textbf{62.35} $\uparrow$\\
			\bottomrule[1.5pt]
	\end{tabular}}\vspace{-5pt}\caption{ID / OOD generalization accuracies on CS-MNIST-COCO. $Z/Z_s$, $Z/Z_c$, and $Z/(Z_s,Z_c)$ indicate different distributional shifts between training and test (without spurious factor, without fake invariant, without the both).}\label{t3}\vspace{-8pt}
\end{table} 
\noindent\textbf{Baselines.} Beyond InvRat and IIB, we take IRM \cite{arjovsky2019invariant} and IRM+IB \cite{ahuja2021invariance} implemented by IRMv1 variants as our IRL baselines. We also employed ERM as the borderline to judge the IRL performance.  

\begin{table}[t]
	\centering
	\setlength{\tabcolsep}{2.5pt}
	{\fontsize{8.5}{12}\selectfont
		\begin{tabular}{c c c c c c  c c ccc  }
			\toprule[1.5pt]
			&\rotatebox{45}{VLCS}& \rotatebox{45}{PACS}&\rotatebox{45}{Office-H}&\rotatebox{45}{Incognita}&\rotatebox{45}{DomainN}&\rotatebox{45}{Average} \cr
			ERM &77.2&83.0&65.7&41.4&40.6&61.6 \\
			IRM &\textbf{78.5}  &83.5&64.3 &\textbf{47.6}&33.9&61.6 \\
			VREx &78.3 &84.9  &66.4 &46.4 &33.6 &61.9\\
			CausalIRL &77.6 &  84.0 &65.7 &46.3 &40.3 &62.8
			\\\hline
			InRav   &77.3&83.5 &66.2 &44.6&35.1&  61.3 \\
			IIB &77.2 &83.9  &68.6 &45.8&41.5&63.4\\ 
			InvRav(+ours) &77.3&84.6 &66.3 &46.1&40.1&62.9  \\
			IIB(+ours) &{77.6}&\textbf{85.8} &\textbf{68.8} &\textbf{47.6}&\textbf{42.5}&\textbf{64.4}  \\
			\bottomrule[1.5pt]
	\end{tabular}}\vspace{-5pt}\caption{OOD generalization accuracy on DomainBed.}\label{t4}\vspace{-6pt}
\end{table} 
\noindent\textbf{Results.} 20\% training data are split into the validation set for CS-MNIST-CIFAR and CS-MNIST-COCO, where all baselines are evaluated to produce their in-distribution (ID) performances. Our evaluation is interested in OOD generalization across diverse distributional shifts: (1).$Z/Z_s$ indicates the test environment with the spurious covariate shift (each test digit was randomly matched with an image drawn from CIFAR or ColorCOCO regardless of the image label); $Z/Z_F$ indicates the test environment with the fake invariant distributional shift (each test digit was randomly matched with a color); $Z/(Z_s,Z_F)$ denotes the test environment containing the spurious and fake invariant distribution shifts concurrently (a test digit takes the both actions simultaneously). 
The accuracies across all baselines are evidenced in Table.\ref{t2} (CS-MNIST-CIFAR) and Table.\ref{t3} (CS-MNIST-COCO).  

Our diagnostic experiment was conducted to address two major concerns to our approach. 1.(\emph{identification concern}): if the fake invariance happens (\emph{i.e.}, $Z_F$ has been activated in the RS-SCM), why IRL needs to identify the fake invariant variable $Z_z$? 2.(\emph{rectification concern}): whether our approach rectify the negative effect caused by the fake invariant features $Z_z$ instead of other spurious covariates? 

In view of the diagnosis concern, we investigate the comparison among diverse testbed scenarios. We first note that the ERM almost perform to approximate the random guess in arbitrary OOD situations, implying our diagnostic benchmark with diverse and significant distributional shifts. Beyond this, $Z/Z_s$ arouses more severe accuracy drop compared with $Z/Z_F$. It makes sense since the introduced image contains spurious covariates with the higher dimensionality than the colorized pixels. But even so, the accuracies of $Z/(Z_s,Z_F)$ in the majority of baselines almost underperform their $Z/Z_s$ counterparts. It verified that the covariate shift caused by the fake invariant variable $Z_{F}$ could not be conveniently eliminated by addressing the shift caused by $Z_s$. It justifies the superiority of our approach, which identifies the covariate shift caused by $Z_F$ to prevent the invariant prediction from the biased representaion $(Z_s,Z_F)$.    

Given this, we compare our approach with other baselines particularly, InvRat and IIB, to verify its rectification ability to the fake invariant factors $Z_F$. In Table.\ref{t2}, the performances of InvRat and IIB boosted by our IIL are inconspicuous in the ID scenarios and OOD scenario $Z/Z_s$ (+2.48 in ID and +1.09 in $Z/Z_s$ for InvRat, respectively; and receives negative transfer in ID and $Z/Z_s$ for IIB), yet its accuracy boost significantly when comes to $Z/Z_F$ and $Z/(Z_s,Z_F)$ scenarios \emph{e.g.}, +3.48 in $Z/Z_F$ and +3.51 in $Z/(Z_s,Z_F)$ for InvRat. In Table.\ref{t3}, the performance boost has been observed more sigficantly. The ablation evidences demonstrate that our rectification approach mainly works for eliminating the negative covariate shift caused by $Z_F$ while thanks to the interplay learning manner, the overall performance get benefited.   

\vspace{-2pt}\subsection{Real-world OOD Generalization}\vspace{-2pt}

\noindent\textbf{Baselines.} We follow the evaluation setup and testify all IRL baselines including ERM, IRM, VREx \cite{krueger2021out}, and the recent approach CausalIRL \cite{chevalley2022invariant}, which all belong to competitive IRL baselines. We also evaluated other 15 baselines apart from IRL approaches in Appendix.C. We employ the model selection strategy by leave-one-domain-out cross validation.  

\noindent\textbf{Results.} In Table.\ref{t4}, we evaluate our approach by combining it with InvRat and IIB. In general, our approach improved InvRav by +1.6\% and IIB by +1.0\%. In terms of our theoretical finding and evidences shown in the diagnostic evaluation, we figure the improvement probably due to the fake variant factors eliminated by our approach. It verified that our approach is compatible with the information bottleneck regularization. Whereas we also observe that the increase in VLCS is very limited (+0.0 for InvRat and +0.4 for IIB in VLCS). It is probably due to VLCS composed by 5 classes across 4 datasets, thus, insufficient to capture the domain-specific complexity. Beyond this, our approach significantly outperform the other baselines.

\vspace{-2pt}\subsection{Ablation}\vspace{-2pt}

\begin{figure}[th]
	\center
	\includegraphics[width = 0.95\columnwidth]{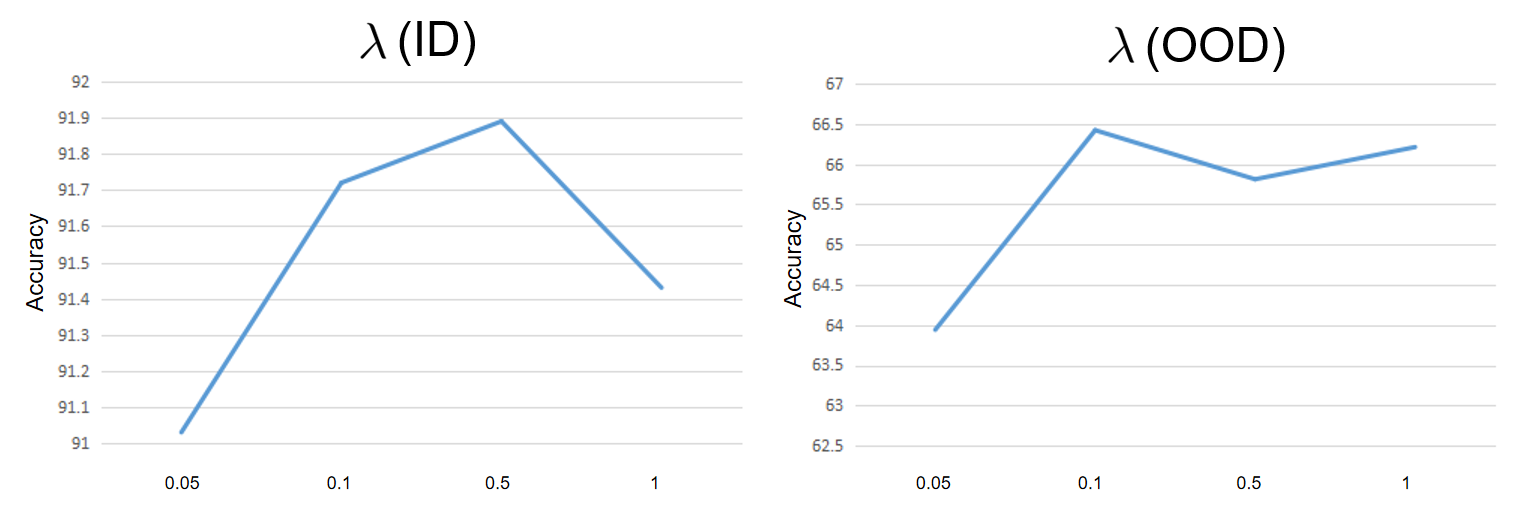}
	\vspace{-6pt}
	\caption{The ablation of $\lambda$ on CS-MNIST-COCO based on InvRat (+ours). The OOD generalization accuracy is based on $Z/Z_F$ setup.
	}
	\vspace{-6pt}
	\label{f1}
\end{figure}
%\vspace{-4pt}\subsection{Diagnostic Experiment}\vspace{-2pt} 
\subsubsection{Ablation of $\lambda$.}
We ablate $\lambda$ in Figure.\ref{f1}(SM). The performances for in-distribution generalization is more robust than OOD generalization as varying $\lambda$. 

\subsubsection{Ablation of neural feature selector $s(\cdot)$.}
We provide the ablation about the role of neural feature selector in the performance boost. We train the IIL model as usual then test the model w/wo neural feature selector, respectively; and we also compare their results with their baseline InvRat. As observed in Table.\ref{t1}, the neural feature selector improve the performance literally based on $h(X)$. We also find that even if the feature selector is deactivated, $h(X)$ trained by IIL remains superior than the original InvRat. 
\begin{table}[h]
	\centering
	\setlength{\tabcolsep}{2.5pt}
	\vspace{-2pt}{\fontsize{9}{15}\selectfont
		\begin{tabular}{c c c c c c  c c c c c  }
			\toprule[1.5pt]
			$s(\cdot)$ &w &wo &InvRat \cr\hline
			Accuracy (ID) &91.72& 90.34 &89.25\cr
			Accuracy (OOD, $Z/Z_F$) &66.42&62.46 &61.75\cr
			\bottomrule[1pt]
	\end{tabular}}\vspace{-2pt}\caption{The ablation of the neural feature selector on CS-MNIST-COCO based on InvRat (+ours).}\label{t1}\vspace{-2pt}
\end{table} 

\subsubsection{Ablation of $I(Z_c;h(X)/Z_c)$.}
We finally provide the ablation of $I(Z_c;h(X)/Z_c)$ executed by MINE technique. As viewed in Table.\ref{t2}, we find that the performance of IIL drop significantly without the MI maximization contraint. It is probably due to the trivial outcome generated by the joint CMI objective. It emphasizes the importance of preventing the trivial solution.   
\begin{table}[h]
	\centering
	\setlength{\tabcolsep}{2.5pt}
	\vspace{-2pt}{\fontsize{9}{15}\selectfont
		\begin{tabular}{c c c c c c  c c c c c  }
			\toprule[1.5pt]
			Methods &InvRat (ID) &IIB (ID) &InvRat (OOD) &IIB (OOD) \cr\hline
			wo &90.15&91.74 &64.26&67.89\cr
			w  &91.72&93.17&66.42&69.59\cr
			\bottomrule[1pt]
	\end{tabular}}\vspace{-2pt}\caption{The ablation of $\lambda$ on CS-MNIST-COCO. The OOD generalization accuracy is based on $Z/Z_F$ setup.}\label{t2}\vspace{-6pt}
\end{table} 

\vspace{-2pt}\section{Conclusion}\vspace{-0pt}
This paper attempts to resolve the fake invariance problem for IRL, which undermines the OOD generalization performance. We proposes a novel structural causal model, ReStructured SCM (RS-SCM) to reconstruct both spurious and fake invariant features from the data. Our RS-SCM inspires a neural feature selection approach based on conditional mutual information to eliminate the spurious and fake invariant effects. Experiments demonstrates the effectiveness of our method on various OOD generalization benchmarks.
%\section{Acknowledgments}

\section{Appendix.A}
\subsection{The Derivation of ``Spuriousness''}
Here we demonstrate the derivation to induce the spuriousness property of invariant features. The equality in the second line is simply derived from MI's chain rule and its rearrangement $I[X;Y|Z]$$=$$I[X;Y,Z]$-$I[X;Z]$, \emph{i.e.},
\begin{small}\vspace{-4pt}
	\begin{displaymath}\begin{aligned}
			I[Y;E|\Phi(Z_{s})]&=I[Y;E,\Phi(Z_{s})]-I[Y;\Phi(Z_{s})]
			\\&=I[Y;E]+ I[Y;\Phi(Z_{s})|E]-I[Y;\Phi(Z_{s})].
			\vspace{-4pt}		\end{aligned}
	\end{displaymath}
	\vspace{-0pt}\end{small}As for the inequality
\begin{displaymath}
	\begin{aligned}
		I[Y;\Phi(Z_{s})]- I[Y;\Phi(Z_{s})|E]
		\geq I[Y;Z_{s}]- I[Y;Z_{s}|E]
	\end{aligned}
\end{displaymath}in the third line, it can be achieved when the shortcut $\Phi(Z)$ is considered as a part of the Markov chain $Y$$\rightarrow$$\Phi(Z_{s})$$\rightarrow$$Z_{s}$ given the Assumptions 1-3. Specifically, we have   
\begin{lem}
	Given the data generation processes defined by FIIF, PIIF, and RS-SCM Assumptions, the shortcut variable $\Phi(Z_{s})$ with the Markov dependency $Y$$\leftarrow$$\Phi(Z_{s})$$\leftarrow$$Z_{s}$ holds the inequality
	\begin{small}
		\begin{displaymath}\begin{aligned}
				I[Y;\Phi(Z_{s})]- I[Y;Z_{s}]\geq I[Y;\Phi(Z_{s})|E]
				- I[Y;Z_{s}|E]
			\end{aligned}
		\end{displaymath}
	\end{small}
\end{lem}
\begin{proof}
	Notice that, the left side of the inequality holds
	\begin{small}\vspace{-4pt}
		\begin{displaymath}\begin{aligned}
				I[Y;\Phi(Z_{s})]- I[Y;Z_{s}]=I[Z_{s};\Phi(Z_{s})|Y]-I[Z_{s};Y|\Phi(Z_{s})], 
			\end{aligned}
		\end{displaymath}
	\end{small}which can be derived from the steps of proving data processing inequality. We may derive the right side of the inequality by following a similar routine: 
	\begin{small}\vspace{-4pt}
		\begin{displaymath}\begin{aligned}
				&I[Y;\Phi(Z_{s})|E]- I[Y;Z_{s}|E]\\=& \Big(\hspace{-0.1em}I[Z_{s};\hspace{-0.1em}\Phi(Z_{s})\hspace{-0.1em},\hspace{-0.1em}Y\hspace{-0.1em}|E]\hspace{-0.2em}-\hspace{-0.2em}I[Y;Z_{s}|E]\Big)\hspace{-0.2em}\\&\ \ \ \ \ \ -\hspace{-0.2em}\Big(I[Z_{s}\hspace{-0.1em};\hspace{-0.1em}\Phi(Z_{s}\hspace{-0.1em}),\hspace{-0.1em}Y|E]\hspace{-0.2em}-\hspace{-0.2em}I[Y\hspace{-0.1em};\hspace{-0.1em}\Phi(Z_{s})|E]\hspace{-0.1em}\Big)
			\end{aligned}
		\end{displaymath}
	\end{small}\vspace{-4pt}and with MI's chain rule, we have
	\begin{small}
		\begin{displaymath}
			\begin{aligned}
				I[Z_{s};\Phi(Z_{s}),Y|E]&=I[Z_{s};\Phi(Z_{s})|E]+I[Z_{s};Y|\Phi(Z_{s}),E]\\&=I[Z_{s};Y|E]+I[Z_{s};\Phi(Z_{s})|Y,E].
			\end{aligned}
		\end{displaymath}
	\end{small}Combining the equations\hspace{-0.0em}, \hspace{-0.1em}the right term\hspace{-0.1em} \hspace{-0.1em} refers\hspace{-0.1em} to\hspace{-0.1em}
	\begin{small}
		\begin{displaymath}\begin{aligned}
				&I[Y;\Phi(Z_{s})|E]\hspace{-0.2em}-\hspace{-0.2em} I[Y;Z_{s}|E]\\=&I[Z_{s};\Phi(Z_{s})|Y,E]\hspace{-0.2em}-\hspace{-0.2em}I[Z_{s};Y|\Phi(Z_{s}),E].
			\end{aligned}
		\end{displaymath}
	\end{small}So the third-line inequality is equivalent with 
	\begin{small}
		\begin{displaymath}\begin{aligned}
				&I[Z_{s};\Phi(Z_{s})|Y]\hspace{-0.2em}-\hspace{-0.2em}I[Z_{s};Y|\Phi(Z_{s})]\hspace{-0.2em}\\\geq&\hspace{-0.2em}I[Z_{s};\Phi(Z_{s})|Y,E]\hspace{-0.2em}-\hspace{-0.2em}I[Z_{s};Y|\Phi(Z_{s}),E].
			\end{aligned}
		\end{displaymath}
	\end{small}We consider the Markov chain $Y$$\leftarrow$$\Phi(Z_{s})$$\leftarrow$$Z_{s}$ recovered by neural nets that reverses the data generation path in SCMs. $Z_{s}$,$Y$ are independent given $\Phi(Z_{s})$ so that $I[Z_{s};Y|\Phi(Z_{s})$ $]$=$0$. Then given $E$ as addition condition \emph{w.r.t.} the SCMs that FIIF, PIIF, and RS-SCM Assumptions suit, the variables $Z_{s}$ and $Y$ are independent so that $I[Z_{s};Y|\Phi(Z_{s}),E]$=$0$. So we only need to justify $I[Z_{s};\Phi(Z_{s})|Y]$$\geq$$I[Z_{s};\Phi(Z_{s})|Y,E]$. It can hold given $I(Z_{s};E|Y)$$>$$0$ and $I(\Phi(Z_{s});E|Y)$$=$$0$, 
	which are both satisfied across the SCMs with the Markov chain $Y$ $\leftarrow$$\Phi(Z_{s})$$\leftarrow$ $Z_{s}$($Z_{s}$,$E$ in the first condition are directly related in the SCMs and for the second condition, since $Z_{s}$ is a collider between $\Phi(Z_{s})$ and $E$, conditioning $Y$ blocks the only backdoor path from $E$ to $\Phi(Z_{s})$, which further leads to the independence). 
\end{proof}We can simply obtain the third line from Lemma.1.
\subsection{Proof of Proposition.1}
\begin{proof}
	The first property of the fake invariant feature $\Phi(Z_s)$ leads to $I[Y;E|Z_{c}\cup\Phi(Z_{s})]$$=$$0$. Given this, we consider the feature $Z$$\in$$\mathbb{R}^{n_c}$ composed by $Z_1$ and $Z_2$, \emph{i.e.}, $Z=[Z_1;Z_2]$, in which $Z_1$ represents the variable of feature subset in $Z_c$ ($Z_1$$\subset$$Z_c$); $Z_2$ denotes the variable of feature subset in $\Phi(Z_s)$ ($Z_2$$\subset$$\Phi(Z_s)$). Derived from the property of $\Phi(Z_{s})$, it leads to $I[Y;E|Z]$$=$$0$ due to $Z$ composed by the invariant and fake invarint parts of $Z_{c}\cup\Phi(Z_{s})$. In terms of $Z_1$$\cup$$Z_2$$\subset$$ $$Z_{c}$$\cup$$\Phi(Z_{s})$, it is also observed that
	\begin{displaymath}
		\begin{aligned}
			I[X;Z]&=I[X;[Z_{1};Z_{2}]]\\
			&=I[X;Z_{1}\cup Z_{2}]\\&\leq I[X;Z_{c}\cup\Phi(Z_{s})].
		\end{aligned}
	\end{displaymath}Therefore we have
	\begin{displaymath}
		\begin{aligned}
			&\lambda I[Y;E|Z]+\beta I[X;Z]\\
			= \ &\lambda I[Y;E|Z_{c}\cup\Phi(Z_{s})]+\beta I[X;Z]\\
			= \ &\leq I[Y;E|Z_{c}\cup\Phi(Z_{s})]+\beta I[X;Z_{c}\cup\Phi(Z_{s})].
		\end{aligned}
	\end{displaymath}The proposition has been proved.
\end{proof}
\subsection{Proof of Proposition.2}
\begin{proof}
	Suppose that $h^\ast(X)$$=$$Z_c\cup Z_F$ and the variable $Z$ denotes a subset of features recovered from $h^\ast(X)$. Given $Z_1$ as the variable of feature subset in $Z_c$ ($Z_1$$\subset$$Z_c$) and $Z_2$ as the variable of feature subset in $Z_F$ ($Z_2$$\subset$$ Z_F$), we consider four cases for $Z$: (1).$Z$$=$$Z_1$;(2).$Z$$=$$Z_2$;(3).$Z$$=$$Z_1$$\cup$$Z_2$;(4).$Z$$=$$\emptyset$. The final case is obviously impossible otherwise we obtain 
	\begin{displaymath}
		I[Y;Z|h^\ast(X)/Z]=I[Y;\emptyset|h^\ast(X)]=0
	\end{displaymath}which violates the condition given. The second case is also impossible since
	\begin{displaymath}
		\begin{aligned}
			I[Y;Z|h^\ast(X)/Z]&=I[Y;Z_2|h^\ast(X)/Z_2]\\
			&=I[Y;Z_2|(Z_c\cup Z_F)/Z_2]
			\\&=I[Y;Z_2|Z_c\cup (Z_F/Z_2)].
		\end{aligned}
	\end{displaymath}In terms of $Z_F$$=$$\Phi(Z_s)$ and $Z_F/Z_2\subset Z_F$, it leads to
	\begin{displaymath}
		\begin{aligned}
			I[Y;Z|h^\ast(X)/Z]=0
		\end{aligned}
	\end{displaymath}which also violates the condition given. 
	
	Here we focus on the third case. It can not assume $Z_2\subsetneqq Z_F$ otherwise
	\begin{displaymath}
		\begin{aligned}
			I[Y;Z|h^\ast(X)/Z]=I[Y;Z|(Z_c/Z_1)\cup(Z_F/Z_2)],
		\end{aligned}
	\end{displaymath}where $Z_F/Z_2$$\neq$$\emptyset$ leads to either
	$$I[Y;Z|(Z_c/Z_1)\cup(Z_F/Z_2)]=0$$ that violates the given condition with $Z_1\subsetneqq Z_c$ or 
	$$I[Y;Z|Z_F/Z_2]=0$$
	that violates the second property of $\Phi(Z_s)$$=$$Z_F$ with $Z_1=Z_c$. 
	In terms of $Z_2=Z_F$ and $Z_1\subsetneq Z_c$, 
	$$I[[Y;Z|h^\ast(X)/Z]=I[Y;Z_1\cup Z_F|Z_c/Z_1]=0$$
	obviously violates the given second condition also. To this, we only have $Z_1=Z_c$ and $Z_2=Z_F$ with
	\begin{displaymath}
		\begin{aligned}
			I[Y;h^\ast(X)/Z|Z]&=I[Y;\emptyset|Z_c\cup Z_F]=0,
			\\I[Y;Z|h^\ast(X)/Z]&=I[Y;Z_c\cup Z_F|\emptyset]\\&=I[Y;Z_c\cup Z_F]>0
		\end{aligned}
	\end{displaymath}that suits the given condition.
	Finally, the first case $Z_1\subsetneqq Z_c$ leads to
	\begin{displaymath}
		\begin{aligned}
			I[Y;Z|h^\ast(X)/Z]=I[Y;Z|(Z_c/Z_1)\cup Z_F]=0
		\end{aligned}
	\end{displaymath}which violates the given second condition. In terms of $Z=Z_1= Z_c$, there exists
	\begin{displaymath}
		\begin{aligned}
			I[Y;h^\ast(X)/Z|Z]&=I[Y;Z_F|Z_c]=0,\\
			I[Y;Z|h^\ast(X)/Z]&=I[Y;Z_c|Z_F]>0,
		\end{aligned}
	\end{displaymath}obivously suiting the given condition also. 
	
	As discussed previously, we obtain the candidates $Z=Z_c$ and $Z=Z_c\cup Z_F$ that both fulfill the propostion's condition implied by the joint CMI objecitve. Whereas the trivial $Z=Z_c\cup Z_F$ would dominate the objective since
	\begin{equation}\label{md}
		\begin{aligned}
			&I[Y;h^\ast(X)/Z|Z]-\lambda I[Y;Z|h^\ast(X)/Z] \ (\textbf{s.t.} \ Z=Z_c)\\
			= \ &I[Y;Z_F|Z_c]-\lambda I[Y;Z_c|Z_F] \\ 
			= \ &I[Y;\emptyset|Z_c\cup Z_F]-\lambda I[Y;Z_c|Z_F] \\ &\ \ \ \ (I[Y;Z_F|Z_c]=I[Y;\emptyset|Z_c\cup Z_F]=0)\\
			> \ &I[Y;\emptyset|Z_c\cup Z_F]-\lambda I[Y;Z_c\cup Z_F|\emptyset]
			\\ &\ \ \ \ (I[Y;Z_c|Z_F]<I[Y;Z_c\cup Z_F|\emptyset])\\
			= \ &I[Y;Z_c/Z_1|Z_1\cup Z_F]-\lambda I[Y;Z_1\cup Z_F|Z_c/Z_1] \\ &(s.t. Z=Z_1\cup Z_F)
		\end{aligned}
	\end{equation}so that we incorporate 
	\begin{equation}\label{mi}
		\begin{aligned}
			I(Z;h(X)/Z)>0\rightarrow\max_{Z}I(Z;h(X)/Z)
		\end{aligned}
	\end{equation}into the joint CMI objective that prevents the result from the trivial solution $Z=h(X)$.
	
	The proposition has been proven.
\end{proof}

\section{Appendix.B}
Our implementation is derived from the code /github.com/ Luodian/IIB/tree/IIB, by simply adding the joint CMI contraint and Eq.\ref{mi} that jointly train our neural feature selector with feature encoder, invariant predictor and domain-aware predictor in IIB. More specifically, in our diagnostic experiments the architectures of feature encoder, invariant predictor and domain-aware predictor exactly follow the architecture used in their CS-MNIST experiment; similarly, our real-world experiment also follows their architecture setup used in the DomainBed experiments. In terms of our neural feature selector $s()$, our diagnostic experiments take a two-layer MLP architecture with a bottleneck \cite{he2016deep}, which follows the width size consistent to the second last layer of the feature encoder; our real-world experiments take a MAB architecture \cite{lee2019set} with a single-head attention deployment. We incorporate the mutual information neural estimator (MINE) \cite{belghazi2018mutual}(https://github.com/mohith-sakthivel/mine-pytorch) to achieve the MI maximization constraint (Eq.\ref{mi}) between $s(h(X))\odot h(X)$ and $(\mathbf{1}-s(h(X)))\odot h(X)$. 

The inner loop training phase of IIL refers to the optimization of InvRat/ IIB based on the code /github.com/Luodian/ IIB/tree/IIB, wherein we follow their hyperparameter setup. In the outer loop, we set up $\lambda=0.1$ across all experiments and use the same optimizer and hyperparameter in the inner loop.  In the alternative phase, we employs the previous optimization setup both for the inner loop and outer loop, then after an outer loop epoch finishes, IIL alters to the inner loop to fine-tune the subnetworks with $10$ iterations until the in-distribution performance converges.

\bigskip
%\noindent Thank you for reading these instructions carefully. We look forward to receiving your electronic files!

\bibliography{aaai24}

\end{document}